\documentclass{article}

\usepackage[accepted]{icml2026}

\usepackage[utf8]{inputenc} 
\usepackage[T1]{fontenc}    
\usepackage{hyperref}       
\usepackage{url}            
\usepackage{booktabs}       
\usepackage{nicefrac}       
\usepackage{microtype}      
\usepackage{xcolor}         

\usepackage{amsmath,amssymb,amsfonts,amsthm}
\usepackage{mathtools}
\usepackage{bm}
\usepackage{bbm}
\usepackage{algorithm}
\usepackage{algorithmic}
\usepackage{graphicx}
\usepackage{multirow}
\usepackage{enumitem}
\usepackage{tcolorbox}
\usepackage{tikz}
\usepackage{natbib}
\usepackage{cleveref}

\hypersetup{colorlinks=true,
  linkcolor=blue!70!black,citecolor=blue!50!black,urlcolor=blue!60!black}
  
\Crefname{figure}{Fig.}{Figs.}
\crefname{figure}{Fig.}{Figs.}

\newtheorem{theorem}{Theorem}

\newtheorem{assumption}{Assumption}
\theoremstyle{definition}

\newtheorem{remark}{Remark}

\newcommand{\cD}{\mathcal D}
\newcommand{\cX}{\mathcal X}
\newcommand{\cC}{\mathcal C}

\providecommand{\E}{\mathbb E}
\providecommand{\Prob}{\mathbb P}
\newcommand{\wh}[1]{\widehat{#1}}

\newcommand{\trainset}{\mathcal D_{\rm tr}}
\newcommand{\valset}{\mathcal D_{\rm val}}
\newcommand{\calset}{\mathcal D_{\rm cal}}
\newcommand{\candset}{\mathcal X_{\rm cand}}

\newcommand{\trueobj}{f}
\newcommand{\surr}{\widehat{f}_{\theta}}
\newcommand{\ww}{w}
\ifdefined\argmin\else\DeclareMathOperator*{\argmin}{arg\,min}\fi
\ifdefined\argmax\else\DeclareMathOperator*{\argmax}{arg\,max}\fi
\newcommand{\Quantile}{\operatorname{Quantile}}
\DeclareMathOperator{\KL}{D_{KL}}

\icmltitlerunning{Conformal Candidate Certification for Offline Model-Based Optimization}

\begin{document}

\twocolumn[
\icmltitle{Conformal Candidate Certification for Offline Model-Based Optimization}

\begin{icmlauthorlist}
\icmlauthor{Seungjin Choi}{croid}
\end{icmlauthorlist}

\icmlaffiliation{croid}{CROID Research and aSSIST University, Seoul, Korea}
\icmlcorrespondingauthor{Seungjin Choi}{seungjin.choi.mlg@gmail.com}

\vskip 0.3in
]

\printAffiliationsAndNotice{}

\begin{abstract}
Offline model-based optimization (MBO) proposes candidates by optimizing a
surrogate trained on a fixed historical dataset. Because candidates are
deliberately out-of-distribution, surrogate rankings are least reliable exactly
where the optimizer is most aggressive, yet existing methods provide no
per-candidate statistical certificate that a design meets a target threshold.
We propose \emph{Conformal Candidate Certification} (CCC), a post-hoc wrapper
that attaches a calibrated one-sided lower bound to each candidate and advances
only those whose bound exceeds the target. We show that entropy-regularized
surrogate maximization induces a Gibbs-tilted proposal, so the same surrogate
supplies importance weights for weighted conformal prediction without a separate
density-ratio estimation step. In a controlled synthetic study, CCC certifies
$16.7\%$ of an aggressive proposal pool with empirical coverage 0.990 at
nominal 0.90, while standard conformal prediction ignoring the covariate shift
collapses to 0.416 coverage.
\end{abstract}

\section{Introduction}

Protein engineering, molecular design, and materials optimization require
selecting designs $x\in\cX$ with high experimental response $\trueobj(x)$
under severe evaluation budgets.  Offline MBO addresses this in a single
shot: given only a static dataset $\cD=\{(x_i,y_i)\}_{i=1}^n$, propose a
small candidate set $\candset$ without further oracle
access~\citep{KumarA2020neurips,TrabuccoB2022icml}.  The usual pipeline
trains a surrogate $\surr$ on $\cD$ and optimizes or conditions on $\surr$
to obtain candidates.

Offline MBO methods intentionally search high-predicted-value regions that
lie away from the historical distribution, where surrogates extrapolate
unreliably.  The result is \emph{surrogate overestimation}: candidates look
excellent under $\surr$ yet perform poorly when measured.
Existing methods~\citep{TrabuccoB2021icml,YuS2021neurips,QiH2022neurips,
KrishnamoorthyS2023icml,MashkariaS2023icml,BrookesDH2019icml},
including conservative surrogate learning (COM, RoMA), invariant
representations (IOM), and conditional generative models (CbAS, DDOM,
BONET), address this by modifying the \emph{proposal} mechanism.
None of them answers the downstream question, ``Given a proposed candidate, does the
calibration data support the claim that $\trueobj(x^*)\ge\tau$ for a target
level $\tau$?'' Proposal (generate high-predicted-value candidates) and
certification (decide which are trustworthy enough to test) are distinct
tasks; conflating them means a conservatism penalty or diffusion-guidance
parameter never translates into a per-candidate coverage guarantee.

We propose \textbf{Conformal Candidate Certification} (CCC), a post-hoc
certification layer that wraps any offline MBO algorithm.
Given any candidate set $\candset$, CCC attaches a conformal lower bound
$\underline{y}(x^*)$ to each candidate and returns
\begin{equation}
  \wh\cC = \{x^*\in\candset: \underline y(x^*)\ge \tau\}.
\end{equation}
Formally, the finite-sample guarantee is for a future measured response
$Y^*$ at the candidate: $\Prob\{Y^*\ge\underline{y}(X^*)\}\ge1-\alpha$.
Under noiseless observations ($Y=\trueobj(X)$) this is equivalent to a
guarantee on the latent objective $\trueobj(x^*)$; under measurement
noise, additional assumptions on the noise model are needed to certify
$\trueobj(x^*)\ge\tau$ directly.
Figure~\ref{fig:overestimation} shows two OOD candidates both overestimated
by the surrogate; CCC certifies the nearer one (smaller uncertainty penalty)
and rejects the farther one (penalty drives the bound below $\tau$), while a
surrogate-threshold rule advances both.

\begin{figure*}[t]
  \centering
  \includegraphics[width=.8\linewidth]{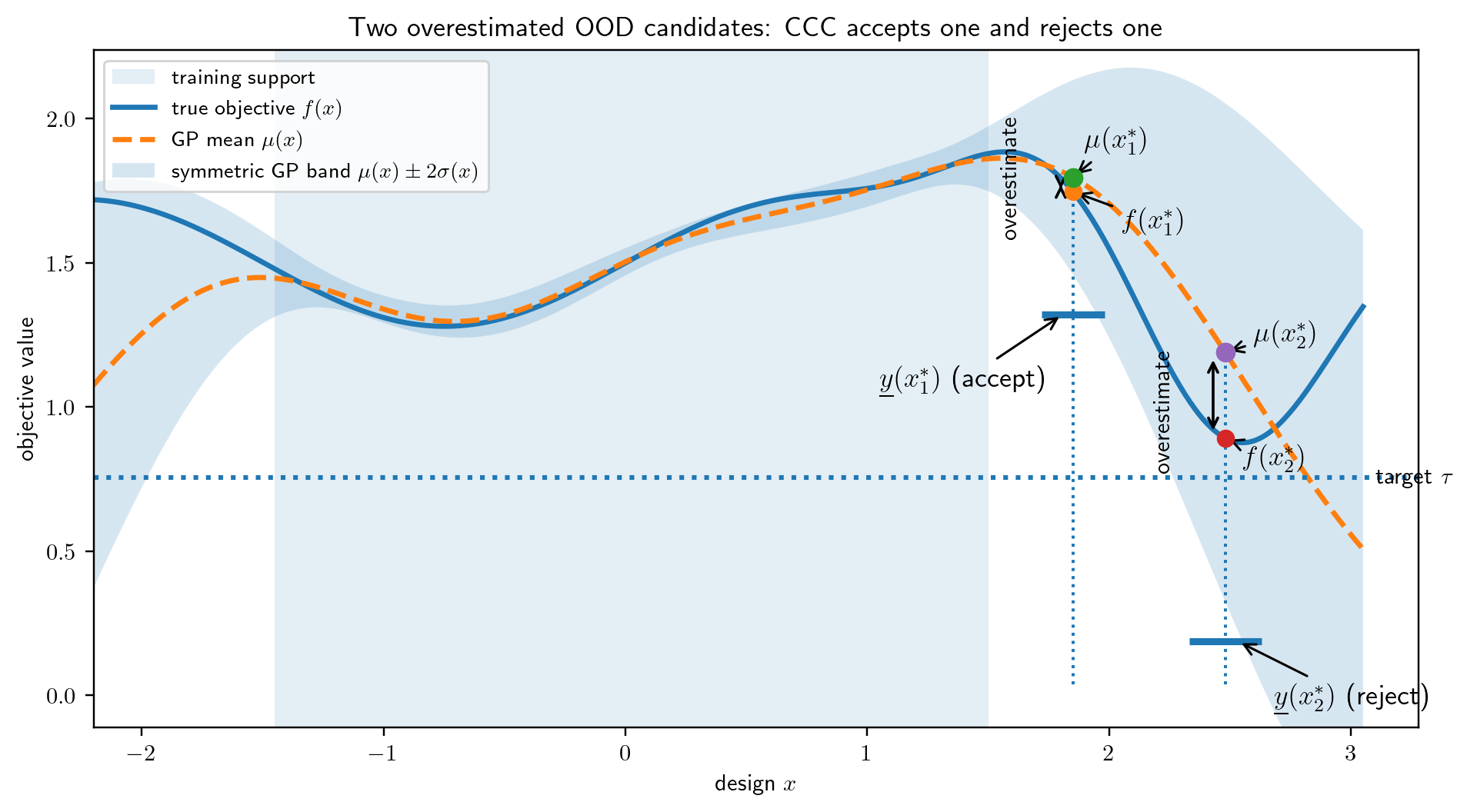}
  \caption{\textbf{Surrogate overestimation and selective certification by CCC.}
  Both candidates lie outside the training support and are overestimated:
  $\mu(x_j^*)>f(x_j^*)$.  A surrogate-threshold rule advances both because
  both GP means exceed $\tau$.  CCC computes $\underline{y}(x_j^*)$ via 
  \eqref{eq:lower-bound} and accepts only if
  $\underline{y}(x_j^*)\ge\tau$.  It accepts $x_1^*$ (smaller conformal
  margin) but rejects $x_2^*$ (larger margin pushes the bound below $\tau$).}
  \label{fig:overestimation}
\end{figure*}

\paragraph{Contributions.} 
\textbf{(i)}~We formulate post-hoc candidate certification for offline MBO
as a covariate-shift conformal prediction problem, separating certification
from proposal for the first time.
\textbf{(ii)}~We model the proposal distribution as a Gibbs tilt
$Q_T\propto P_X\exp\{\surr(x)/T\}$, derived from entropy-regularized
surrogate maximization; the same surrogate that creates the shift supplies
the importance weights, eliminating a separate density-ratio estimation step.
\textbf{(iii)}~We prove finite-sample marginal lower-bound validity under
oracle weights: $\Prob\{Y^*\ge\underline{y}(X^*)\}\ge 1-\alpha$ for a
fresh candidate from the proposal distribution.  The required
score-measurability condition is verified directly from the
strict data-splitting discipline of Assumption~\ref{ass:noleak}.
\textbf{(iv)}~A synthetic study shows CCC certifies $16.7\%$ of an
aggressive Boltzmann proposal pool with near-zero false certifications
and empirical coverage $0.990$ at nominal $0.90$, while Unweighted CP
collapses to $0.416$ coverage under the Boltzmann covariate shift.

A closely related framework is \emph{conformalized selection}
\citep{JinY2023jmlr,JinY2026biometrika}, which constructs conformal
$p$-values controlling the false discovery rate over a selected set
under (weighted) exchangeability.
CCC differs in two ways: it targets a \emph{per-candidate} lower
certificate $\underline{y}(x^*)$ rather than an FDR rule, and it
models the covariate shift as algorithm-induced via the Gibbs tilt
rather than requiring an external density-ratio estimate.
The two frameworks are complementary: FDR control over $\wh\cC$ can
be obtained by passing CCC's conformal $p$-values
$\wh p_{m+1}^w(x_j^*)$ into the weighted conformalized selection
procedure of \citet{JinY2026biometrika}.

\section{Problem Setup}
 
Let $P$ denote the historical data-generating distribution over $(X,Y)$,
with design marginal $P_X$ and conditional response $P(Y\,|\, X)$.  The
offline dataset $\cD$ is partitioned into three disjoint splits: $\trainset$
for surrogate fitting and proposal, $\valset$ for estimating plug-in weights,
and $\calset=\{(x_i,y_i)\}_{i=1}^m$ for final
conformal calibration.  An offline MBO algorithm, seeing only $\trainset$,
returns a candidate set $\candset=\{x_j^*\}_{j=1}^K$, which we model as
approximate draws from a proposal distribution $Q_X$ over designs.
 
\begin{assumption}[Covariate shift]
\label{ass:covshift}
The conditional response $P(Y\,|\, X=x)$ is the same for historical and
candidate designs.  Only the design marginal shifts from $P_X$ to $Q_X$.
\end{assumption}
 
\begin{assumption}[No calibration leakage]
\label{ass:noleak}
The proposal algorithm uses only $\trainset$.  All plug-in quantities
(the temperature $\wh T$ and importance weights) are fixed from $\trainset$,
$\valset$, the calibration covariates $\{X_i:i\in\calset\}$, and the
candidate covariates before any response in $\calset$ is accessed.
\end{assumption}

\begin{assumption}[Bounded oracle weights]
\label{ass:bounded}
The oracle density ratio $\ww(x)=dQ_X/dP_X(x)$ exists and satisfies
$\ww(x)\le w_{\max}<\infty$ for $P_X$-almost all $x$.
\end{assumption}
 
The certification target is a lower bound $\underline{y}(X^*)$ satisfying,
for a fresh candidate $X^*\sim Q_X$ and response $Y^*\sim P(\cdot\,|\, X^*)$,
\begin{equation}
  \Prob\{Y^*\ge \underline y(X^*)\}\ge 1-\alpha.
  \label{eq:target-validity}
\end{equation}
The guarantee \eqref{eq:target-validity} is for a single fresh candidate
$X^*\sim Q_X$; it does not imply that the accepted set $\wh\cC$ has a
bounded false-certification rate.  Finite-sample FDR control over $\wh\cC$
can be obtained by treating $\wh p^w_{m+1}(x_j^*)$ as a conformal $p$-value
for each candidate and applying the weighted conformalized selection
procedure of \citet{JinY2026biometrika}.

\section{Conformal Candidate Certification}

Conformal prediction provides distribution-free uncertainty quantification by
calibrating conformal scores on held-out data
\citep{VovkV2005book,AngelopoulosAN2023ftml}.
Standard split conformal prediction requires exchangeability between
calibration and test points, but offline MBO breaks this: calibration inputs
are drawn from the historical distribution $P_X$ while candidates are drawn
from the proposal distribution $Q_X$.  Importance-weighted conformal
prediction (IW-CP) restores marginal validity when the density ratio
$\ww = dQ_X/dP_X$ is known \citep{TibshiraniR2019neurips}.  The rest of
this section builds the complete CCC pipeline: derive $\ww$, estimate its
temperature, form the IW atoms, construct the one-sided score, and
invert the weighted quantile to obtain the lower bound.

\subsection{Gibbs-tilt proposal model and importance weights}

The proposal distribution should concentrate on high-surrogate designs while
staying close to $P_X$, where the surrogate is calibrated.  This yields an entropy-regularized optimization:
\begin{equation}
  Q_T = \argmax_{Q\ll P_X}
  \Bigl\{\E_{X\sim Q}\bigl[\surr(X)\bigr] - T\,\KL(Q\|P_X)\Bigr\},
\end{equation}
whose closed-form solution is the Gibbs tilt
\begin{eqnarray}
  q_T(x) & = &  \frac{p_X(x)\exp\{\surr(x)/T\}}{Z_T},  \nonumber \\  
  Z_T & = &  \E_{P_X}\exp\{\surr(X)/T\},
  \label{eq:gibbs}
\end{eqnarray}
giving density ratio $w(x) = \exp\{\surr(x)/T\}/Z_T$.
The identity \eqref{eq:gibbs} is exact for Boltzmann or softmax proposals
and serves as an interpretable working model for gradient-ascent or
generative MBO.  Its key property is that $Z_T$ cancels in the normalized
IW-CP atoms (below), so the surrogate alone determines the weights.

\paragraph{Temperature estimation.}
We estimate $T$ by moment-matching: find $\wh T$ such that the tilted
mean surrogate score over a reference set equals the mean candidate score
$\bar f_{\rm cand}=K^{-1}\sum_{j=1}^K\surr(x_j^*)$.
Using $\valset$ alone as the reference can severely underestimate $T$
when candidates are OOD relative to $\valset$: the tilt collapses onto
the few high-$\surr$ validation points, producing degenerate weights and
near-vacuous bounds.  We therefore pool $\valset$ with the candidate
surrogate scores to anchor the upper tail of the reference:
\begin{equation}
  \wh T = \argmin_{T\in[T_{\min},T_{\max}]}
  \left(
    \frac{\sum_{r\in\mathcal{R}}\exp\{\surr(r)/T\}\surr(r)}
         {\sum_{r\in\mathcal{R}}\exp\{\surr(r)/T\}}
    - \bar f_{\rm cand}
  \right)^{\!2},
  \label{eq:temp-est}
\end{equation}
where $\mathcal{R}=\valset\cup\{\surr(x_j^*)\}_{j=1}^K$ pools the
validation set with all candidate surrogate scores, and
$[T_{\min},T_{\max}]$ (e.g.\ $[0.01,100]$ on the surrogate scale) is
searched by bisection.  Candidate scores are always available at
estimation time and do not include calibration responses, so
Assumption~\ref{ass:noleak} is satisfied.
When the proposal temperature is known (e.g.\ an explicit Boltzmann
proposal), $\wh T$ may be set to the oracle value directly.

\paragraph{IW conformal atoms.}
 For a candidate $x^*$, IW-CP places atoms over the $m$ calibration
points and the (unknown) test point:
\begin{eqnarray}
  p_i^w(x^*) & = & \frac{\ww(x_i)}{\sum_{\ell=1}^m\ww(x_\ell)+\ww(x^*)},
  \nonumber \\
  p_{m+1}^w(x^*) & = & \frac{\ww(x^*)}{\sum_{\ell=1}^m\ww(x_\ell)+\ww(x^*)}.
\end{eqnarray}
Substituting the Gibbs ratio, $Z_T$ cancels and the practical atoms are
\begin{equation}
  \wh p_i^w(x^*) =
  \frac{\exp\{\surr(x_i)/\wh T\}}
       {\sum_{\ell=1}^m\exp\{\surr(x_\ell)/\wh T\}+\exp\{\surr(x^*)/\wh T\}},
  \label{eq:gibbs-atoms}
\end{equation}
with an analogous expression for $\wh p_{m+1}^w(x^*)$.  Each candidate
induces its own augmented weighted empirical distribution.  For generative
MBO methods whose proposal is poorly described by \eqref{eq:gibbs}, a
discriminative fallback fits a classifier to distinguish $\valset$ from
$\candset$ and uses the odds-ratio identity
\citep{SugiyamaM2012book}.

\subsection{One-sided score}

We use a surrogate trained on $\trainset$, with predicted mean $\mu(x)$
and uncertainty scale $\sigma(x)>0$.  Since certification is a one-sided
problem (we advance $x^*$ only if $\trueobj(x^*)$ is certifiably above
$\tau$), a symmetric score wastes calibration probability on the upper
tail.  The nonconformity score is the signed one-sided residual
\begin{equation}
  s(x,y) = \frac{\mu(x)-y}{\sigma(x)},
  \label{eq:one-sided}
\end{equation}
which is large when the surrogate overestimates.  Inverting $s(x,y) \le q$
gives directly $y\ge\mu(x)-q\sigma(x)$.

\begin{remark}[Surrogate and uncertainty choice]
For GP surrogates, $\mu(x)$ and $\sigma(x)$ are the predictive mean and
standard deviation.  For non-GP surrogates (e.g.\ gradient boosting),
set $\sigma(x)=1$, reducing the score to the signed one-sided residual
$\mu(x)-y$.  The conformal quantile absorbs the scale, so the lower
bound remains valid.
\end{remark}

A natural extension is to replace $\mu(x)$ with a proposal-weighted
local surrogate mean that adjusts the score center to reflect the
proposal's concentration near $x^*$; we set this correction to zero
throughout (the construction and conditions under which it helps are
discussed in the supplementary material of an extended version).

\subsection{Weighted conformal quantile and lower bound}

For a fixed candidate $x^*$, compute calibration scores
\[
  s_i(x^*) = s(x_i, y_i) = \frac{\mu(x_i)-y_i}{\sigma(x_i)},
  \qquad i\in\calset,
\]
and the IW quantile
\begin{eqnarray}
\wh q_{1-\alpha}^w(x^*)  & = &   \Quantile_{1-\alpha}\! \left(
    \sum_{i=1}^m \wh p_i^w(x^*)\,\delta_{s_i(x^*)} \right. \nonumber \\
   &  & +  \, \wh p_{m+1}^w(x^*)\,\delta_{+\infty} \Big).
  \label{eq:weighted-quantile}
\end{eqnarray}
The $+\infty$ atom accounts for the unknown test score.
Whenever $\wh p_{m+1}^w(x^*)>\alpha$, the finite-mass part of the
distribution sums to less than $1-\alpha$, so
$\wh q_{1-\alpha}^w(x^*)=+\infty$ and the lower bound is vacuous:
the correct behavior when the calibration data cannot support
certification of $x^*$.

Since $s(x^*,y)$ is strictly decreasing in $y$, inverting
$s(x^*,y^*)\le\wh q_{1-\alpha}^w(x^*)$ gives the CCC lower bound:
\begin{equation}
  \underline{y}(x^*) =
  \mu(x^*) - \wh q_{1-\alpha}^w(x^*)\,\sigma(x^*).
  \label{eq:lower-bound}
\end{equation}
The two terms are transparent: surrogate center and conformal uncertainty
penalty.  A candidate passes when $\underline{y}(x^*)\ge\tau$.

\begin{algorithm}[t]
\caption{Conformal Candidate Certification (CCC)}
\begin{algorithmic}[1]
\REQUIRE Dataset $\cD$, candidate set $\candset$, level $\alpha$,
         threshold $\tau$.
\STATE Split $\cD$ into $\trainset$, $\valset$, $\calset$.
\STATE Train surrogate on $\trainset$; obtain $\mu$, $\sigma$, and $\surr$.
\STATE Estimate $\wh T$ on $\valset$ by \eqref{eq:temp-est}.
\FOR{each $x^*\in\candset$}
  \STATE Compute Gibbs atoms $\wh p_i^w(x^*)$ and
         $\wh p_{m+1}^w(x^*)$ by \eqref{eq:gibbs-atoms}.
  \STATE Compute calibration scores
         $s_i(x^*)=s(x_i,y_i)$, $i\in\calset$.
  \STATE Compute $\wh q_{1-\alpha}^w(x^*)$ by \eqref{eq:weighted-quantile}.
  \STATE Compute $\underline{y}(x^*)$ by \eqref{eq:lower-bound}.
\ENDFOR
\STATE \textbf{Return} $\wh\cC=\{x^*\in\candset:\underline{y}(x^*)\ge\tau\}$.
\end{algorithmic}
\end{algorithm}

\section{Validity and Diagnostic Properties}

Theorem~\ref{thm:validity} applies the importance-weighted conformal
prediction framework of \citet{TibshiraniR2019neurips} to the offline
MBO setting.  The contribution of CCC lies in how this framework is
instantiated: the algorithm-induced covariate shift is governed by the
Gibbs-tilt proposal $Q_T$, which supplies analytically tractable
importance weights from the surrogate itself, and the one-sided score
$s(\cdot,\cdot)$ \eqref{eq:one-sided} can be inverted into a lower
certificate on the objective value.  The theorem makes precise the
validity guarantee that CCC inherits when these weights are oracle and
the score is fixed before the calibration responses are accessed.

\begin{theorem}[Oracle-weight marginal lower-bound validity]
\label{thm:validity}
Suppose Assumptions~\ref{ass:covshift}, \ref{ass:noleak}, and
\ref{ass:bounded} hold, and suppose the weighted conformal quantile
uses the true density ratio $\ww=dQ_X/dP_X$.  Let
$\underline{y}^{\rm oracle}(x^*)$ denote the lower bound
\eqref{eq:lower-bound} computed with these oracle weights.  Then for a
fresh candidate $X^*\sim Q_X$ and response $Y^*\sim P(\cdot\,|\, X^*)$,
\begin{equation}
  \Prob\{Y^*\ge \underline{y}^{\rm oracle}(X^*)\}\ge 1-\alpha.
\end{equation}
Under noiseless observations $Y=\trueobj(X)$, the same bound holds with
$\trueobj(X^*)$ in place of $Y^*$.
\end{theorem}

Theorem~\ref{thm:validity} should be read carefully.  It guarantees
validity for the \emph{oracle} lower bound $\underline{y}^{\rm oracle}(x^*)$
computed with the true density ratio.  The practical CCC algorithm uses
plug-in Gibbs weights, so the implemented $\underline{y}(x^*)$ does not
inherit the exact finite-sample guarantee.  The theorem is marginal over
a fresh candidate $X^*\sim Q_X$: it does not hold simultaneously for all
$K$ proposed candidates, does not imply post-selection validity
conditional on $x^*\in\wh\cC$, and requires $Q_X\ll P_X$.

\begin{remark}[Plug-in weights and empirical coverage]
When Gibbs weights are estimated from data, the finite-sample guarantee
no longer holds exactly.  Empirical coverage above the nominal level
in repeated experiments is consistent with the oracle target, but should
be interpreted as diagnostic evidence rather than a finite-sample guarantee.
\end{remark}

\begin{remark}[Noisy measurements vs.\ latent objective]
Theorem~\ref{thm:validity} certifies the \emph{future measured response}
$Y^*$, not the latent objective $\trueobj(X^*)$.  If the scientific
target is $\trueobj(x^*)\ge\tau$, additional assumptions (e.g.\ known
noise variance) are needed.  In the noiseless case $Y=\trueobj(X)$ the
two coincide.
\end{remark}

\section{Experiments: Synthetic Validation}

We present a controlled synthetic stress test designed to address the main
failure mode in offline MBO: an aggressive surrogate proposal can rank
overestimated OOD candidates above the experimental target, and a useful
certification rule should select a nontrivial subset rather than return
the empty set.

\subsection{Setup}

The training support is \([-1.5,1.5]\), while candidates are proposed in the
OOD region \([1.55,3.0]\).  The true objective is
\begin{eqnarray}
  f(x)
  & = &
  1.15
  +0.65\exp\!\left\{-\frac{(x-1.85)^2}{2(0.28)^2}\right\} \nonumber \\ 
  & - & 0.80\exp\!\left\{-\frac{(x-2.58)^2}{2(0.32)^2}\right\}
  +0.05\sin(3x). \nonumber \\
\end{eqnarray}
Thus the OOD region contains a genuinely high-value neighborhood near
\(x\approx1.85\) and an overestimated low-value neighborhood near
\(x\approx2.58\).  We define an intentionally optimistic surrogate center
\begin{eqnarray}
  \mu(x)
  & = & 
  f(x)
  +0.08\sin(4x)
  +0.14(x-1.5)_+ \nonumber \\
  & + & 1.25\exp\!\left\{-\frac{(x-2.58)^2}{2(0.38)^2}\right\},
  \label{eq:selective-surrogate}
\end{eqnarray}
with GP-like uncertainty scale
\begin{equation}
  \sigma(x)
  =
  0.08
  +0.22(x-1.5)_+
  +0.12\exp\!\left\{-\frac{(x-2.60)^2}{2(0.35)^2}\right\}.
  \label{eq:selective-sigma}
\end{equation}
The target threshold is \(\tau=1.45\).  The surrogate center exceeds \(\tau\)
throughout much of the OOD region, including points whose true objective falls
below \(\tau\).  This creates a nontrivial screening problem: a naive surrogate
threshold rule over-accepts, while a useful CCC rule should certify only the
subset whose lower certificates clear the target.

Candidates are sampled from a Gibbs proposal proportional to
\(\exp\{\mu(x)/T_{\rm prop}\}\) with \(T_{\rm prop}=0.35\) over
\([1.55,3.0]\), which produces a mixture of high-value and overestimated
OOD candidates.  Since the proposal temperature is known by construction,
we set \(\wh T=T_{\rm prop}=0.35\) directly; moment-matching on
\(\valset\) underestimates \(T\) and renders most bounds vacuous in this
setup.  Auxiliary data are split into \(\valset\) and \(\calset\); the
calibration split is reserved for the final weighted conformal quantile.
We use \(\alpha=0.10\) for the main comparison, \(K=300\) candidates per
seed, and average over \(40\) independent seeds.

\begin{table*}[t]
\centering
\caption{Selective-certification experiment (\(\tau=1.45\), \(\alpha=0.10\),
\(K=300\), 40 seeds).  CCC certifies a nontrivial subset of the aggressive
proposal pool while maintaining empirical lower-bound coverage.}
\label{tab:selective-main}
\small
\setlength{\tabcolsep}{2.5pt}
\begin{tabular}{lccccc}
\toprule
Method & Pass & Cov. & Prec. & False & Mean \(f\) \\
\midrule
Naive Surrogate & \(1.000{\pm}0.000\) & -- & \(0.410{\pm}0.029\) & \(0.590{\pm}0.029\) & \(1.150{\pm}0.023\) \\
Unweighted CP   & \(0.500{\pm}0.108\) & \(0.416{\pm}0.300\) & \(0.773{\pm}0.126\) & \(0.227{\pm}0.126\) & \(1.555{\pm}0.035\) \\
CCC             & \(0.167{\pm}0.099\) & \(0.990{\pm}0.055\) & \(0.996{\pm}0.025\) & \(0.004{\pm}0.025\) & \(1.688{\pm}0.019\) \\
\bottomrule
\end{tabular}
\vspace{-1mm}
\end{table*}

We compare three rules: (1) \emph{\bf Naive Surrogate} accepts candidates
with \(\mu(x^*)\ge\tau\); (2) \emph{\bf Unweighted CP} applies standard
split conformal with equal weights, ignoring the Boltzmann covariate
shift; (3) \emph{\bf CCC} uses the IW-weighted conformal quantile with
the one-sided score \eqref{eq:one-sided}.  We report pass rate, empirical
lower-bound coverage, precision, false-acceptance rate, and mean true
objective among accepted candidates.

\subsection{Selective Certification Results}

Table~\ref{tab:selective-main} shows CCC performs selective certification
rather than trivial abstention.  The naive surrogate rule accepts all
candidates yet only $41\%$ are truly above \(\tau\).  The key finding is
that Unweighted CP, which ignores the Boltzmann covariate shift,
dramatically under-covers: empirical coverage is $0.416$, far below the
nominal $0.90$, with a $22.7\%$ false-acceptance rate.  This confirms
that the Gibbs-weight correction is essential when candidates are drawn
from a shifted distribution.  CCC certifies $16.7\%$ of the pool with
empirical coverage $0.990$ and false-acceptance rate below $0.5\%$.

To distinguish CCC from a merely conservative abstention rule, we
evaluate coverage across multiple nominal levels.
Table~\ref{tab:selective-calibration} reports empirical coverage for
\(\alpha\in\{0.05,0.10,0.15,0.20\}\).  CCC maintains empirical coverage
at or above the nominal level while certifying a nontrivial fraction of
candidates at every \(\alpha\).

\begin{table}[t]
\centering
\caption{Calibration of CCC across nominal levels (40 seeds,
same run as Table~\ref{tab:selective-main}).}
\label{tab:selective-calibration}
\small
\setlength{\tabcolsep}{3pt}
\begin{tabular}{ccccc}
\toprule
\(\alpha\) & Nominal & Emp. cov. & Pass & Prec. \\
\midrule
0.05 & 0.95 & \(1.000{\pm}0.000\) & \(0.003{\pm}0.008\) & \(1.000{\pm}0.000\) \\
0.10 & 0.90 & \(0.990{\pm}0.055\) & \(0.167{\pm}0.099\) & \(0.996{\pm}0.025\) \\
0.15 & 0.85 & \(0.964{\pm}0.152\) & \(0.237{\pm}0.098\) & \(0.986{\pm}0.063\) \\
0.20 & 0.80 & \(0.938{\pm}0.174\) & \(0.285{\pm}0.095\) & \(0.975{\pm}0.072\) \\
\bottomrule
\end{tabular}
\vspace{-1mm}
\end{table}

Figure~\ref{fig:selective-calibration} visualizes the same calibration
diagnostic for CCC.  CCC is conservative but not vacuous: increasing \(\alpha\) raises
the pass rate while empirical lower-bound coverage remains above the nominal
line.

\begin{figure}[t]
\centering
\includegraphics[width=.9\linewidth]{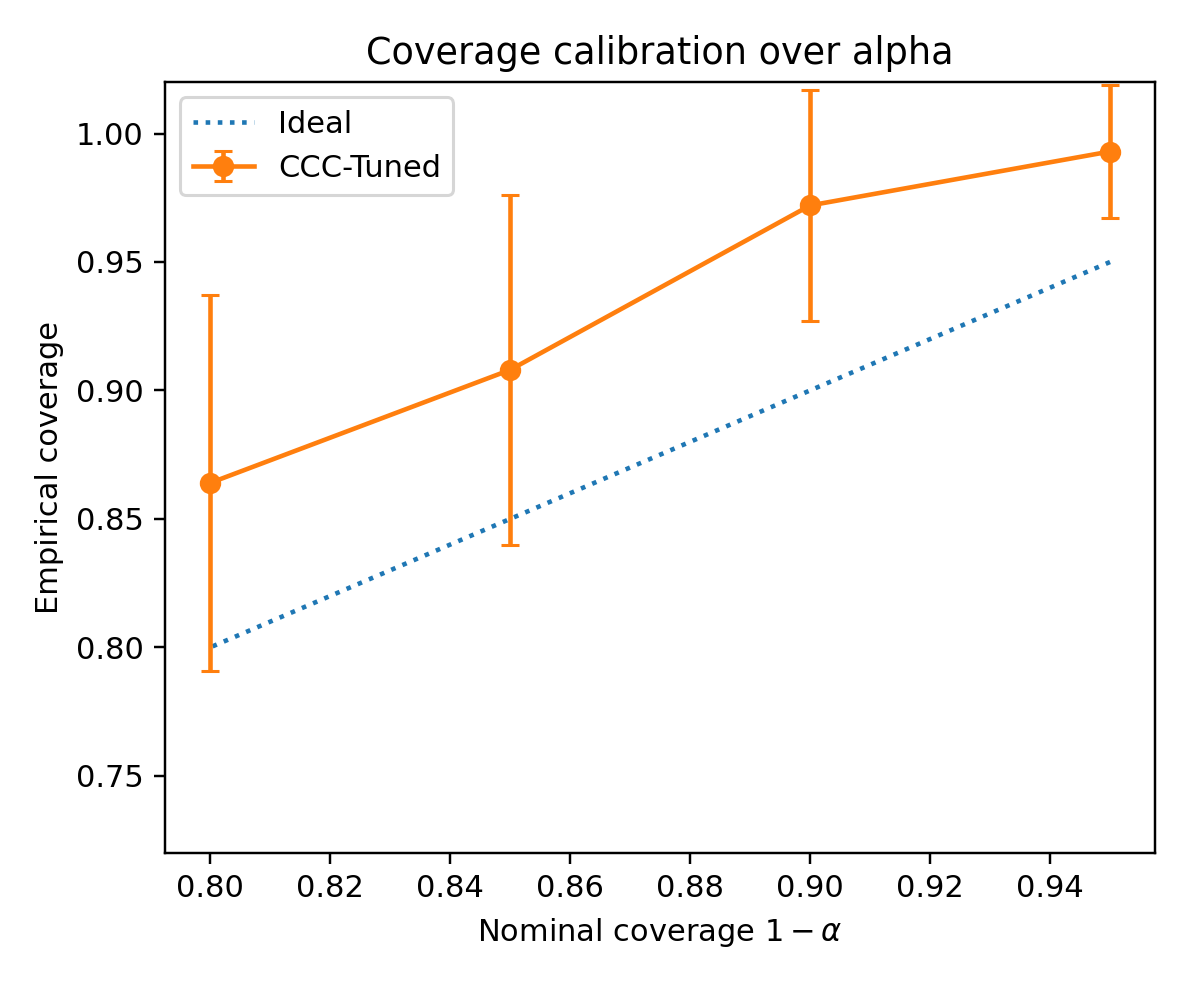}
\caption{Coverage calibration of CCC across
\(\alpha\in\{0.05,0.10,0.15,0.20\}\).  The dotted line is ideal
calibration.  CCC maintains empirical coverage at or above the nominal
level while certifying a nontrivial subset at each \(\alpha\).}
\label{fig:selective-calibration}
\end{figure}

\section{Discussion and Conclusion}
 
CCC separates offline MBO into two distinct tasks: proposal and
certification.  Any proposal algorithm can remain aggressive;
CCC then determines which candidates have enough empirical support
to justify costly evaluation, and the accepted set $\wh\cC$ seeds
the next experimental round.
 
The main limitation is weight estimation.  The finite-sample guarantee
assumes oracle weights; the practical Gibbs tilt is exact only for
entropy-regularized proposals and approximate otherwise.  The failure
mode is conservative rather than anti-conservative: underestimating
$T$ widens the conformal penalty and certifies fewer candidates while
maintaining coverage.  High-risk discovery may still require a separate
exploration budget for uncertified candidates.
 
Future work includes coverage-degradation bounds under plug-in weights,
set-level FDR control by combining CCC's per-candidate $p$-values with
\citet{JinY2026biometrika}, and online updates after certified
experimental rounds.  CCC provides a principled post-hoc answer to the
question offline MBO leaves unanswered: which proposed candidates are
trustworthy enough to test?  Under oracle weights the answer carries
marginal conformal validity; with estimated weights it gives an
empirically testable certification layer for offline-to-online
decision-making.

\bibliographystyle{icml2026}
\bibliography{/users/seungjin/pub/bib/sjc}

\clearpage
\onecolumn
\appendix

\section{Proof of Theorem~\ref{thm:validity}}
\begin{proof}
Let
\[
  \calset = \{(X_i,Y_i)\}_{i=1}^m
\]
be the final calibration split, with
\[
  (X_i,Y_i)\sim P_X(dx)P(dy\,|\,x), \qquad i=1,\ldots,m,
\]
and let
\[
  (X^*,Y^*)\sim Q_X(dx)P(dy\,|\,x)
\]
be an independent fresh candidate-response pair.  By
Assumption~\ref{ass:noleak} and the construction of the CCC pipeline,
the candidate-specific score map
\[
  (x,y)\mapsto s(x,y)
\]
is fixed before the final calibration responses are used: $\mu$,
$\sigma$, and $\wh{T}$ are all determined from $\trainset$, $\valset$,
candidate covariates, and calibration covariates only, with no dependence
on any $Y_i\in\calset$.

Under Assumption~\ref{ass:covshift}, the calibration and candidate
distributions differ only in their design marginals: the conditional
response distribution $P(Y\,|\,X)$ is the same.  Under
Assumption~\ref{ass:bounded}, the oracle density ratio
\[
  w(x) = \frac{dQ_X}{dP_X}(x)
\]
exists and is finite.  Applying the covariate-shift conformal theorem of
\citet{TibshiraniR2019neurips} to this fixed candidate-specific score
gives the following marginal statement.  With oracle atoms
\[
  p_i^w(X^*)
  =
  \frac{w(X_i)}{\sum_{\ell=1}^m w(X_\ell)+w(X^*)},
  \qquad
  p_{m+1}^w(X^*)
  =
  \frac{w(X^*)}{\sum_{\ell=1}^m w(X_\ell)+w(X^*)},
\]
and weighted quantile $q_{1-\alpha}^w(X^*)$ formed from
\[
  \sum_{i=1}^m p_i^w(X^*)\,\delta_{s(X_i,Y_i)}
  + p_{m+1}^w(X^*)\,\delta_{+\infty},
\]
we have
\[
  \Prob\!\left\{
    s(X^*,Y^*)
    \le q_{1-\alpha}^w(X^*)
  \right\}
  \ge 1-\alpha.
\]
This probability is marginal over the calibration sample and the fresh
candidate-response pair.  It is not a conditional coverage statement for
every realized set of covariates.

It remains only to invert the one-sided score.  Since $\sigma(X^*)>0$,
\[
  s(X^*,Y^*) \le q_{1-\alpha}^w(X^*)
\]
is equivalent to
\[
  Y^* \ge
  \mu(X^*)
  - q_{1-\alpha}^w(X^*)\,\sigma(X^*)
  = \underline{y}^{\rm oracle}(X^*).
\]
Therefore
\[
  \Prob\{Y^*\ge\underline{y}^{\rm oracle}(X^*)\}\ge 1-\alpha.
\]
If the experiment is noiseless, $Y=\trueobj(X)$ almost surely, the same
argument gives
\[
  \Prob\{\trueobj(X^*)\ge\underline{y}^{\rm oracle}(X^*)\}\ge 1-\alpha.
\]
\end{proof}

\end{document}